\documentclass{article}
\usepackage{spconf, amsmath, amsfonts, amsthm, graphicx}
\usepackage{booktabs} 
\usepackage{bm,  algorithm, algorithmic}

\newcommand{\E}{\text{E}}
\renewcommand{\P}{\text{P}}
\newcommand{\I}{\mathrm{I}}

\newtheorem{theorem}{Theorem}
\newtheorem{corollary}{Corollary}[theorem]

\title{Sequential Maximum Margin Classifiers for Partially Labeled Data}
\name{Elizabeth Hou, Alfred O. Hero \thanks{This work was partially supported by the Consortium for Verification Technology under Department of Energy National Nuclear Security Administration award number DE-NA0002534 and partially by the University of Michigan ECE Departmental Fellowship.}}
\address {University of Michigan \\ Dept. of Electrical Engineering and Computer Science \\
1301 Beal Avenue, Ann Arbor, MI 48109-2122 }

\begin{document}
\ninept
\maketitle
\begin{abstract}
In many real-world applications, data is not collected as one batch, but sequentially over time, and often it is not possible or desirable to wait until the data is completely gathered before analyzing it. Thus, we propose a framework to sequentially update a maximum margin classifier by taking advantage of the Maximum Entropy Discrimination principle. Our maximum margin classifier allows for a kernel representation to represent large numbers of features and can also be regularized with respect to a smooth sub-manifold, allowing it to incorporate unlabeled observations. We compare the performance of our classifier to its non-sequential equivalents in both simulated and real datasets.
\end{abstract}
\section{Introduction}
\label{sec:intro}

As the popularity of big data increases and more data is being gathered, the importance of sequential models that are able to continuously update with new data has increased. These models are particularly crucial in high throughput real-time applications such as speech or streaming text classification. To this end, we propose a sequential framework to update the probabilistic maximum margin classifier built from the Maximum Entropy Discrimination (MED) principle of \cite{NIPS1999_1733}.

The proposed sequential MED framework can be cast as recursive Bayesian estimation where the likelihood function is a log-linear model formed from a series of constraints and weighted by Lagrange multipliers. In the Gaussian case it shares similarities with the problem of Gaussian process classification, which has been previously studied \cite{wahba1999support, jaakkola1999probabilistic, smola1998connection, opper1999gaussian, Sollich2002, Rasmussen:2005:GPM:1162254}, but to the best of our knowledge, a method to recursively update the Gaussian process classifier has not been developed. In the single time point case, sequential MED can be specialized to the support vector machine \cite{smola1998connection} and Laplacian support vector machine \cite{Belkin:2006:MRG:1248547.1248632} as previously discussed in \cite{NIPS1999_1733} and \cite{hou}.

We are interested in situations where we receive a stream of data $ \bm{X}_{(1)}, \bm{X}_{(2)}, \ldots$ over time $t$ where each $X_{(t)}$ is a matrix of dimension $p \times n$, with $p$ denoting the number of feature variables and $n$ denoting the number of i.i.d. samples, where $n=n_{(t)}$ may vary with time. In the fully labeled scenario, the data has corresponding labels  $y_i = [1, -1] \, \forall i \text{ and } t$; however in the partially labeled scenario, at each time point $t$, only $ l_{(t)} < n_{(t)} $ of the samples have labels. We define the observed data at any time point $t$ as $\mathcal{D}_{(t)} =  \{\bm{X}_{(t)}, \bm{y}_{(t)} \}$ and all observed data up to time $\tau$ as $ \{ \mathcal{D}_{(t)} \}_{t=1}^\tau$. Such scenarios would arise in a variety of domains such as a satellite that only transmits its data daily or a government agency that only releases its data quarterly with their corresponding reports. The rest of the paper is organized as follows: Section 2 and Section 3 will discuss how to sequentially update the corresponding MED models for supervised and semi-supervised classification. Section 4 validates the method by simulation and we present an application to a dataset of spoken letters of the English alphabet.

\section{Sequential MED} \label{sec:SeqMED}

Constrained relative entropy minimization is used to estimate the closest distribution to a given prior distribution subject to a set of moment constraints. The authors of \cite{koyejo2013representation} show that, if the prior distribution is from the exponential family, then the density that optimizes the constrained relative entropy problem is also a member of the exponential family.  Similar to Bayesian conjugate priors, there exist relative entropy conjugate priors that facilitate evaluation of the closest distribution. These produce optimal constrained relative entropy densities, which can be thought of as posteriors, from the same parametric family as the prior. Maximum entropy discrimination (MED) \cite{NIPS1999_1733} also admits conjugate priors as it a special case of constrained relative entropy minimization where one of the constraints is over a parametric family of discriminant functions $ \mathcal{L}(\bm{X} | \bm{\Theta}) $.

\subsection{Review of MED for Maximum Margin Classification}

In this paper, we are interested in maximum margin binary classifiers. In this case the discriminant function $ \mathcal{L}(\bm{X} | \bm{\theta}, b) = f(\bm{X}) \bm{\theta} + b  $ is linear for some feature transformation $f(\cdot)$, feature weights vector $\bm{\theta}$, and bias term $b$. Slack variables $\gamma_i$ are used to create a margin in the constraints $\E(y_i ( f(\bm{X}_i) \bm{\theta} + b) - \gamma_i)$, the expected hinge loss with slack variables.  The MED objective function is
\begin{flalign} \label{MED_obj}
&\underset{\P(\bm{\Theta}, \bm{\gamma} |  \mathcal{D} )}{\min} \text{KL}\left(\P(\bm{\Theta}, \bm{\gamma} | \mathcal{D} || \P_0(\bm{\Theta}, \bm{\gamma}) \right) \qquad \text{subject to} \\
&\iint \P(\bm{\Theta}, \bm{\gamma}|  \mathcal{D} ) \, (y_i ( f(\bm{X}_i) \bm{\theta} + b) - \gamma_i) \, d\bm{\Theta} d\bm{\gamma} \ge 0   \,\, \forall i = 1, \dots, n \notag
\end{flalign}
whose solution $\P(\bm{\Theta}, \bm{\gamma} | \mathcal{D}) $ is the constrained minimum relative entropy posterior. The associated MED decision rule $ \hat{y}_{i'} = \text{sgn}( \iint \P(\bm{\Theta}|\mathcal{D})  (f(\bm{x}_{i'}) \bm{\theta} + b) \,  d\bm{\Theta} ) $ is a weighted combination of discriminant functions. The minimum relative entropy posterior has the form
$$ \P(\bm{\Theta}, \bm{\gamma}| \mathcal{D}) = \frac{\P_0(\bm{\Theta}, \bm{\gamma})}{Z(\bm{\alpha})} \exp \left\{ \sum_{i=1}^n \alpha_i\left (y_i (f(\bm{X}) \bm{\theta} + b) - \gamma_i \right) \right\} $$
where $ \bm{\alpha} = [ \alpha_1, ..., \alpha_n ]^T \ge 0 $ are Lagrange multipliers that minimize the partition function $Z(\bm{\alpha})$. It is common to set the initial prior distribution to the separable form: \\ $ \P_0(\bm{\Theta}, \bm{\gamma}) = \P_0(\bm{\theta}) \P_0(b) \prod_{i=1}^n \P_0(\gamma_i )  $. If in addition, we specify that $ \P_0(\gamma_i) = C e^{-C(1-\gamma_i)} \mathcal{I}(\gamma_i \le 1) $, $ \P_0(\bm{\theta}) $ is $ N(\bm{0}, \bm{\I}) $, and $ \P_0(b) $ is a zero mean Bayesian non-informative (diffuse) prior, denoted $N(0, \infty)$, then the Lagrange multipliers can be obtained as the solution $\hat{\bm{\alpha}}$ to the constrained optimization
\begin{flalign*} 
&  \underset{\bm{\alpha}}{\max} -\frac{1}{2} \bm{\alpha}^T \bm{Y} f(\bm{X}) f(\bm{X})^T \bm{Y} \bm{\alpha} + \sum_{i=1}^n \alpha_i + \log(1-\alpha_i /C) \\
& \text{subject to } \sum_{i=1}^n y_i \alpha_i = 0 \text{ and } \alpha_1, \dots, \alpha_n \ge 0 
\end{flalign*}
where $\bm{Y} = \text{diag}(\bm{y})$. This objective function has a log barrier term $ \log(1-\alpha_i/C) $ instead of the inequality constraints $ \alpha_i  \leq C $ commonly found in the dual form of the SVM. Except in some ill-defined cases where the maximum lies near the boundary of the feasible set, the $ \hat{\alpha}_i $ will be identical to the optimal support vectors that maximize the SVM objective. The authors in \cite{NIPS1999_1733, hou} show that the \textit{maximum a posteriori} (MAP) estimator for $\bm{\theta}$ of the MED posterior is related to the Lagrange multipliers by $ \hat{\bm{\theta}} = f(\bm{X})^T \hat{\bm{\alpha}} $, so the MED posterior mode is equivalent to a maximum margin classifier. 

\subsection{Updating MED}

Under the separable prior assumptions above, the MED posterior $ \P(\bm{\Theta}, \bm{\gamma}| \mathcal{D}) $ will take the factored form $\P(\bm{\theta}| \mathcal{D}) \P(b| \mathcal{D}) \P(\bm{\gamma}) $. Due to the fact that the slack parameters $\gamma_i$ do not depend on the data $\mathcal{D}$, the density $\P(\bm{\gamma})$ does not affect the MED decision rule given after \eqref{MED_obj}. Hence only $\P(\bm{\theta}| \mathcal{D}) $ and $\P(b| \mathcal{D}$ are important. This  remaining part of the MED posterior has the form: $ \P(\bm{\theta}| \mathcal{D}) \P(b| \mathcal{D}) = N( f(\bm{X})^T \bm{Y} \bm{\alpha} , \bm{\I}) N(0, \infty)$, which is a conjugate distribution. Due to this conjugacy the posterior distribution optimizing the objective in \eqref{MED_obj} can be propagated forward in time in a recursive manner. The updating procedure is given in the following theorem and corollaries.

\begin{theorem} \label{thm:SeqMED}
Let the MED prior at $t = 1$ be $ \bm{\theta} \sim N(\bm{0}, \bm{\I}),  b \sim N(0, \infty)$, and $\emph{P}_0(\gamma_i) = C_{(1)} e^{-C_{(1)} (1-\gamma_i)} \mathcal{I}(\gamma_i \le 1) $. Then given data $ \mathcal{D}_{(\tau)}$ at time point $\tau$, the relative entropy conjugate priors are
\begin{flalign*}
& \emph{P}_0 \left(\bm{\theta} | \{ \mathcal{D}_{(t)} \}_{t=1}^{\tau-1} \right) = N\left( \sum_{t=1}^{\tau-1} f(\bm{X}_{(t)})^T \bm{Y}_{(t)}  \hat{\bm{\alpha}}_{(t)} , \bm{\I}\right) \\
& \emph{P}_0\left(b | \{ \mathcal{D}_{(t)} \}_{t=1}^{\tau-1} \right) = N(0, \infty)  \\
& \emph{P}_0(\bm{\gamma}) = \prod_{i=1}^{n_{(\tau)}} C_{(\tau)} \exp \left\{-C_{(\tau)} (1-\gamma_i) \right\} \mathcal{I}(\gamma_i \le 1)
\end{flalign*}
and the MED posterior $ \emph{P}(\bm{\Theta}| \left\{ \mathcal{D} \right\}_{t=1}^{\tau} ) $ can represented as
$$ \emph{P} \left(\bm{\theta} | \left\{ \mathcal{D} \right\}_{t=1}^{\tau} \right) = N\left( \bm{\mu}_0 + f(\bm{X}_{(\tau)})^T \bm{Y}_{(\tau)}  \hat{\bm{\alpha}}_{(\tau)}, \bm{\I} \right) $$
where $ \bm{\mu}_0 = \sum_{t=1}^{\tau-1} f(\bm{X}_{(t)})^T \bm{Y}_{(t)}  \hat{\bm{\alpha}}_{(t)} $ is the prior mean and $ \emph{P}(b | \left\{ \mathcal{D} \right\}_{t=1}^{\tau} ) $ is the same as the Bayes non-informative prior.
\end{theorem}
 
Introducing the kernel function $ k(\bm{x}, \bm{x}') = \langle f(\bm{x}), f(\bm{x}')\rangle $ and the parameter transformation $ \bm{\omega} =  f(\bm{X}) \bm{\theta} $, the posterior at time $\tau>0$ can be represented in terms of this kernel. 
\begin{corollary} \label{coll:kern}
The equivalent prior at $t = 1$ for the transformed parameter is $ \bm{\omega} \sim N(\bm{0}, \bm{K}_{(1)})$ where $\bm{K}_{(1)} =  f(\bm{X}_{(1)}) f(\bm{X}_{(1)})^T $. Furthermore, the posterior at time $\tau$ is of Gaussian form \\ $ \emph{P}(\bm{\omega} | \{\mathcal{D}_{(t)}\}_{t=1}^\tau) = N(\bm{\mu}_{(\tau)}, \bm{K}_{(\tau)} ) $ where the mean parameter satisfies the recursions
$\bm{\mu}_{(\tau)} = \bm{\mu}_{(\tau-1)} + \bm{K}_{(\tau)} \bm{Y}_{(\tau)} \hat{\bm{\alpha}}_{(\tau)}  $.
\end{corollary}

Since $\P(\bm{\theta}| \{ \mathcal{D}_{(t)} \}_{t=1}^{\tau}) $ is Gaussian, the MAP estimator is simply the mean parameter $\bm{\mu}_{(\tau)}$ given in the Corollary \ref{coll:kern}. Thus the decision rule reduces to $ \hat{y}_{i'} = \text{sgn}(f(\bm{x}_{i'}) \hat{\bm{\theta}} + \hat{b}) $ where the MAP estimator $ \hat{\bm{\theta}}  $ is a function of the previously estimated Lagrange multipliers $\hat{\bm{\alpha}}_{(1)}, \dots, \hat{\bm{\alpha}}_{(\tau-1)}$ and the maximizing values $\hat{\bm{\alpha}}_{(\tau)}$ and $\hat{b}$ for the current time point $\tau$.

\begin{corollary} \label{coll:Opt_Lagrange}
Given all previous $\hat{\bm{\alpha}}_{(1)}, \dots, \hat{\bm{\alpha}}_{(\tau-1)}$, the current optimal Lagrange multipliers $\hat{\bm{\alpha}}_{(\tau)} $ are the solution to
\begin{flalign*} 
&  \underset{\bm{\alpha}_{(\tau)}}{\max} \, -\frac{1}{2} \bm{\alpha}_{(\tau)}^T \bm{Y}_{(\tau)} \bm{K}_{(\tau)} \bm{Y}_{(\tau)} \bm{\alpha}_{(\tau)} + \sum_{i=1}^{n_{(\tau)}} \log ( 1-\alpha_{(\tau) i}/ C_{(\tau)}) \\
& \hspace{34pt} + \bm{\alpha}_{(\tau)}^T \left(\bm{1} - \bm{Y}_{(\tau)} \sum_{t=1}^{\tau-1} k(\bm{X}_{(\tau)} , \bm{X}_{(t)}) \bm{Y}_{(t)} \hat{\bm{\alpha}}_{(t)} \right)  \\
& \text{subject to }\bm{y}_{(\tau)}^T  \bm{\alpha}_{(\tau)} = 0 \text{ and } \alpha_{(\tau) i} \, \ge 0 \text{ for all } i = 1, \dots, n_{(\tau)} 
\end{flalign*}
and, holding the Lagrange multipliers fixed, the optimal bias $\hat{b} = $
$$
\underset{b}{\arg\min} \hspace{-6pt} \sum_{s \in \{i | \hat{\alpha}_{(\tau) i} \neq 0 \} } \left\vert \left(y_{(\tau) s} - \sum_{t=1}^{\tau} k(\bm{X}_{(\tau) s} , \bm{X}_{(t)} ) \bm{Y}_{(t)}  \hat{\bm{\alpha}}_{(t)}  \right) - b \right\vert
$$
ensures that the expectation constraints in the objective hold. 
\end{corollary}

The above dual formulation for the Lagrange multipliers $ \bm{\alpha}_{(\tau)}$ has some interesting implications. Since the Lagrange multipliers from the previous time points are fixed at time step $\tau$, the factor $ {\bm{1} - \bm{Y}_{(\tau)} \sum_{t=1}^{\tau-1} k(\bm{X}_{(\tau)} , \bm{X}_{(t)}) \bm{Y}_{(t)} \hat{\bm{\alpha}}_{(t)} }$ are constants and can be thought of as (unnormalized) weights for $ \bm{\alpha}_{(\tau)}$, the Lagrange multipliers from the current time point. Thus the corresponding Lagrange multipliers for samples that are easily predicted using only the prior information will have lower weight than the Lagrange multipliers for samples that are difficult or incorrect.  

\section{Manifold Regularization}

Next we consider the case wheres some of the labels are missing. Without loss of generality we will assume the first $l$ points are labeled and the latter $n-l$ points are unlabeled.

We will adopt the semi-supervised MED classification framework of \cite{hou}, called Laplacian MED (LapMED). LapMED introduces an additional ``geometric" constraint
\begin{flalign} \label{laplace_const}
\hspace{-5pt} \iint \P(\bm{\theta}, \lambda) \left( \int_{x \in \mathcal{M}} \hspace{-2pt} \bm{\theta}^T \hspace{-2pt} f(\bm{x}) \Delta_{\mathcal{M}} f(\bm{x}) \bm{\theta}  \, d\mathcal{P}_x - \lambda \right) d\bm{\theta} d\lambda \leq 0
\end{flalign}
to \eqref{MED_obj} where $ \mathcal{M} = \text{supp}(\mathcal{P}_X ) \subset \mathbb{R}^n $ is a compact submanifold, $ \Delta_{\mathcal{M}} $ is the Laplace-Beltrami operator on $\mathcal{M}$,  and $ \lambda $ controls the complexity of the decision boundary in the intrinsic geometry of $ \mathcal{P}_X $. This constraint was motivated by the semi-supervised framework of \cite{Belkin:2006:MRG:1248547.1248632} to encourage the function $f(x)$ to be smooth over the support set of the feature distribution $\mathcal{P}_X$, inducing a geometric interpolation of unlabeled points. Since the marginal distribution is unknown, from \cite{grigor2006heat}
$$
f(\bm{X})^T \bm{L} f(\bm{X}) \rightarrow \int_{x \in \mathcal{M}} f(\bm{x})  \Delta_{\mathcal{M}} f(\bm{x}) \, d\mathcal{P}_x,  \, \text{ as } n\rightarrow \infty
$$
where $ \bm{L} $ is the normalized graph Laplacian formed with a heat kernel. The LapMED posterior can be approximated as
\begin{flalign*} 
& \P(\bm{\theta}, b, \bm{\gamma}, \lambda | \mathcal{D}) = \frac{\P_0(\bm{\theta}, b, \bm{\gamma}, \lambda)}{Z(\bm{\alpha}, \beta)}  \exp \Bigg\{ \\
& \sum_{i=1}^l \alpha_i \left( y_i ( f(\bm{X}) \bm{\theta} + b) - \gamma_i \right) + \beta \left(\lambda - \bm{\theta}^T f(\bm{X})^T \bm{L} f(\bm{X}) \bm{\theta}  \right) \Bigg\} 
\end{flalign*}
where $ \beta \ge 0 $ is a Lagrange multiplier for the smoothness constraint. 

\subsection{Sequential Laplacian MED}

The distribution $ \P(\bm{\Theta}, \bm{\gamma}, \lambda| \mathcal{D}) $ that minimizes the objective with the additional constraint \eqref{laplace_const} can similarly be factorized and, like the distribution of slack parameters considered in Section 2, the distribution of the smoothness parameter $\lambda$ is also independent of the data $\mathcal{D}$.  Likewise, the distribution of the decision rule coefficients $ \P(\bm{\Theta} | \mathcal{D} )$  are conjugate distributions with their priors. Thus the updating procedure for the LapMED problem is similar to the updating procedure in Section \ref{sec:SeqMED}.

\begin{theorem} \label{thm:SeqLapMED}
At $t = 0$, the MED priors for $\bm{\theta}$ (or $\bm{\omega}$), $b$, and $\gamma_i$ are the same as in Theorem 1, and the prior for $\lambda$ is a Bayesian zero mean point prior, denoted $Exp.(\infty)$.
Then given data $\mathcal{D}_{(\tau)}$ at time point $\tau$, the MED conjugate prior and posterior are still $Exp.(\infty)$ for $\lambda$, the same as in Theorem 1 for $b$ and $\gamma_i$, and Gaussian of form $N \left(\bm{\mu}_{(\tau)}, \bm{\Sigma}_{(\tau)} \right)$ for $\bm{\theta}$ (or $\bm{\omega}$). Define a $ l \times n $ expansion matrix as  $\bm{J} = [ \bm{\I} \,\, \bm{0} ] $. Then the mean and covariance parameters for the distribution of $\bm{\theta}$ are
$$ 
\bm{\mu}_{(\tau)} = \bm{G}_{(\tau)}^{-1} \sum_{t=1}^{\tau} f(\bm{X}_{(t)})^T \bm{J}^T  \bm{Y}_{(t)}  \hat{\bm{\alpha}}_{(t)}, \quad \bm{\Sigma}_{(\tau)} = \bm{G}_{(\tau)}^{-1}, 
$$
where $ \bm{G}_{(\tau)} = \bm{G}_{(\tau-1)} + 2\beta_{(\tau)} f(\bm{X}_{(\tau)})^T \bm{L}_{(\tau)} f(\bm{X}_{(\tau)}) $ is a recursive graph of vertex disjoint subgraphs, and for the distribution of $\bm{\omega}$ are
$$
\bm{\mu}_{(\tau)} = \hspace{-2pt} \sum_{t=1}^{\tau} k_{(\tau)} \hspace{-2pt} \left( \bm{X}_{(\tau)}, \bm{X}_{(t)} \right) \hspace{-1pt} \bm{J}^T \bm{Y}_{(t)}  \hat{\bm{\alpha}}_{(t)}, \, \bm{\Sigma}_{(\tau)} = k_{(\tau)} \hspace{-2pt} \left( \bm{X}_{(\tau)}, \bm{X}_{(\tau)} \right) 
$$
where $k_{(\tau)}( \bm{x}, \bm{x}') = \langle f(\bm{x}), \bm{G}_{(\tau)}^{-1} f(\bm{x}') \rangle$ is a kernel function that can be recursively defined as
\begin{flalign} \label{kern_func} 
& k_{(\tau)}( \bm{x}, \bm{x}') = k_{(\tau-1)}(\bm{x}, \bm{x}') - k_{(\tau-1)}(\bm{x}, \bm{X}_{(\tau)}) \bigg( \hspace{-2pt} \left(2 \beta_{(\tau)}\bm{L}_{(\tau)}\right)^{-1} \notag \\
& \hspace{42pt} + k_{(\tau-1)} \left(\bm{X}_{(\tau)}, \bm{X}_{(\tau)} \right) \hspace{-2pt} \bigg)^{-1} \hspace{-2pt} k_{(\tau-1)}( \bm{X}_{(\tau)}, \bm{x}') .
\end{flalign}
\end{theorem}
 
 Theorem 2 gives the posterior distribution for semi-supervised classification whose form is comparable to the form given in Corollary \ref{coll:kern} for the supervised case. Indeed the forms are identical except for the presence of the precision matrix term $G_{(\tau)}$ in the semi-supervised case. As the sparsity of $G_{(\tau)}$ is associated with the graph Laplacian, the kernel function of the semi-supervised case is a regularized version of the kernel function that appears in Corallary \ref{coll:kern}. If we let $ \beta_{(t)} $ be a fixed parameter, then $\hat{\bm{\alpha}}_{(t)}$ and $\hat{b}$ optimize an objective of the same form as in Corollary \ref{coll:Opt_Lagrange}, but with kernel function $k_{(\tau)}( \bm{x}, \bm{x}')$. If $ \beta_{(t)} $ is chosen to be 0, the sequential LapMED simply ignores the unlabeled data of time point $t$, and if all $\beta_{(i)}$'s are $0$, then the unlabeled data is always ignored and the updating procedure is exactly the same as in the supervised scenario. These parameters are functions of the $\gamma_A$ and $\gamma_I$, which are identical to the penalty parameters in the Laplacian SVM \cite{Belkin:2006:MRG:1248547.1248632}, associated with the reproducing kernel Hilbert space and data distribution respectively: $ C_{(t)}  = \frac{1}{2 l_{(t)} \gamma_A } $ and $ \beta_{(t)}  = \frac{\gamma_I }{2 \gamma_A n_{(t)}^2} $.

\subsection{Approximating the Kernel Function} \label{approx_k_func}

Because the kernel function in \eqref{kern_func} is a function of the previous kernel functions, calculating a map to its associated Hilbert space $ \mathcal{H}_{(\tau)} $ can be computationally expensive. Thus in this subsection, we derive an approximation to the map to $ \langle \bm{x}, \bm{x}' \rangle_{\mathcal{H}_{(\tau)}}  $, which is computationally easier than direct recursive calculation.

Recall that we approximate the constraint in \eqref{laplace_const}, at any time point $t$, empirically with the graph Laplacian $\bm{L}_{(t)}$ formed using the data from that time point $\bm{X}_{(t)}$. However, the non-empirical constraint using the Laplace-Beltrami operator over the unknown marginal distribution $\mathcal{P}_x$, is actually the same at every time point. Thus as $n_{(\tau-1)} \rightarrow \infty$,  the prior graph $\bm{G}_{(\tau-1)}$ converges to 
\begin{flalign} \label{decomp_Lap} 
B \int_{x \in \mathcal{M}} f(\bm{x}) \Delta_{\mathcal{M}} f(\bm{x}) \, d\mathcal{P}_x \approx B \sum_{i=1}^{\infty} \delta_i \xi_i^2 \upsilon_i(z) \upsilon_i(z)
\end{flalign}
where $B = 2 \sum_{t=1}^{\tau} \beta_{(t)} $, $\delta_i $ are the eigenvalues of the Laplace-Beltrami operator, and $\upsilon_i(z)$ and $\xi_i$ are the infinite sequence of right singular functions and singular values of $ f(x) = \int k(x, z) f(z) \, dz $. The approximate decomposition arises since the left singular functions of $f$ are the eigenfunctions of the Laplace-Beltrami operator \cite{Lederman} and \cite{Belkin:2006:MRG:1248547.1248632}. Thus instead of empirically approximating the Laplacian as a sum of subgraphs \\ $\bm{G}_{(\tau-1)} = \bm{\I} + \sum_{t=1}^{\tau-1} 2\beta_{(t)} f(\bm{X}_{(t)})^T \bm{L}_{(t)} f(\bm{X}_{(t)}) $, we can instead implement approximations to the eigen/singular values and singular functions in \eqref{decomp_Lap}. 

Assuming that the sample size $n$ is large enough, the average eigenvalues of the $\tau-1$ graph Laplacians would be a good estimator for the eigenvalues of the Laplace-Beltrami operator. Additionally the rows of the matix $\bm{V}^T$ from the singular value decomposition of $\bm{X}$ will contain the basis for its row space. Thus because the right singular functions form an orthonormal basis for the coimage of $f$, if the mapping approximately preserves the basis, the mapped average singular vectors $f(\bar{\bm{V}}_i)$ would be good estimators for the right singular functions $\upsilon_i(z)$ and correspondingly so for the singular values.

The posterior kernel function $k_{(\tau)}( \bm{x}, \bm{x}')$ using an approximation to the decomposition in \eqref{decomp_Lap} will no longer be a recursive function of prior kernel functions $k_{(\tau-1)}( \bm{x}, \bm{x}')$ that have the same form, like in \eqref{kern_func}. Instead for $\tau > 2$, it uses a prior kernel function 
\begin{flalign*} 
& \tilde{k}_{(\tau-1)}(\bm{x}, \bm{x}')  = k(\bm{x}, \bm{x}') - k(\bm{x}, \bar{\bm{V}}_{(\tau-1)}) \bigg( \hspace{-2pt} \frac{\text{ diag}(\bar{\bm{s}}_{(\tau-1)}^{\,2} \bar{\bm{d}}_{(\tau-1)})^{-1}}{B} \\
& \hspace{52pt} + k (\bar{\bm{V}}_{(\tau-1)}, \bar{\bm{V}}_{(\tau-1)}) \bigg)^{-1} \hspace{-2pt} k( \bar{\bm{V}}_{(\tau-1)}, \bm{x}') .
\end{flalign*}
where $ k(\bm{x}, \bm{x}') = \langle f(\bm{x}), f(\bm{x}')\rangle $ is the non-regularized kernel function. So at time $\tau$, the singular vectors of $\bm{X}_{(\tau-1)}$ are used to update the average singular vectors, in the above function, through
$$ \bar{\bm{V}}_{(\tau-1)} = \bar{\bm{V}}_{(\tau-2)} + \frac{\bm{V}_{(\tau-1)} - \bar{\bm{V}}_{(\tau-2)}}{\tau-1} $$ and similarly so for the average corresponding singular values $ \bar{\bm{s}}_{(\tau-1)}$ and the average eigenvalues of the graph Laplacians $ \bar{\bm{d}}_{(\tau-1)}$.

\section{Experiments}

In this section, we compare the proposed sequential maximum margin classifiers to popular supervised and semi-supervised maximum margin classifiers (SVM \cite{smola1998connection} and LapSVM \cite{Belkin:2006:MRG:1248547.1248632}) where the model is trained using just the current time points data and where the model has been re-trained on all previous data. The former type of model is a lower bound on performance since it ignores all previous data and the latter type of model is an upper bound since it is re-trained on all previous data at every time point. Note the MED and SVM models only differ by a weak log-barrier term in the objective function making their performance identical, and similarly so for LapMED and LapSVM. Thus their performance curves will referred to as Full SVM/MED and Full LapSVM/LapMED.

\subsection{Simulations}

In both of the following simulations, the models receive roughly 100 samples ($n_{(t)} = [97, 103] $) at every time point, the parameters are empirically chosen with a validation set, and then the models are tested on an independent data set of 1000 test points. The test accuracy $\frac{TP + TN}{1000}$ is the average accuracy over 100 trials of simulation.

In the first simulation, we generate data from 200 categorical distributions where 100 of the variables are sparse so they have high probability of being 0, another 50 of the variables have lower probability of being 0, and the final 50 variables are used to distinguish between the two classes. We use the term frequency - inverse document frequency (TF-IDF) kernel of \cite{elkan2005deriving}, which is used in document processing and topic models. Figure \ref{fig:super} shows that the accuracy of the sequential model (SeqMED) improves as the model is updated with more training data and has much better results even after one model update versus the independent model (SVM) that ignores previous training data. Of course the sequential model does not improve as rapidly as the model that is re-trained on all the data (Full SVM/MED), but this is the price paid for lower computational complexity. For example, at $t = 30$, SeqMED updates and fits 100 coefficients for the new data whereas Full SVM/MED fits 3,000 coefficients for all the data.

\begin{figure}[htb]
	\begin{minipage}[b]{1.0\linewidth}
		\centering
		\centerline{\includegraphics[width=8.5cm]{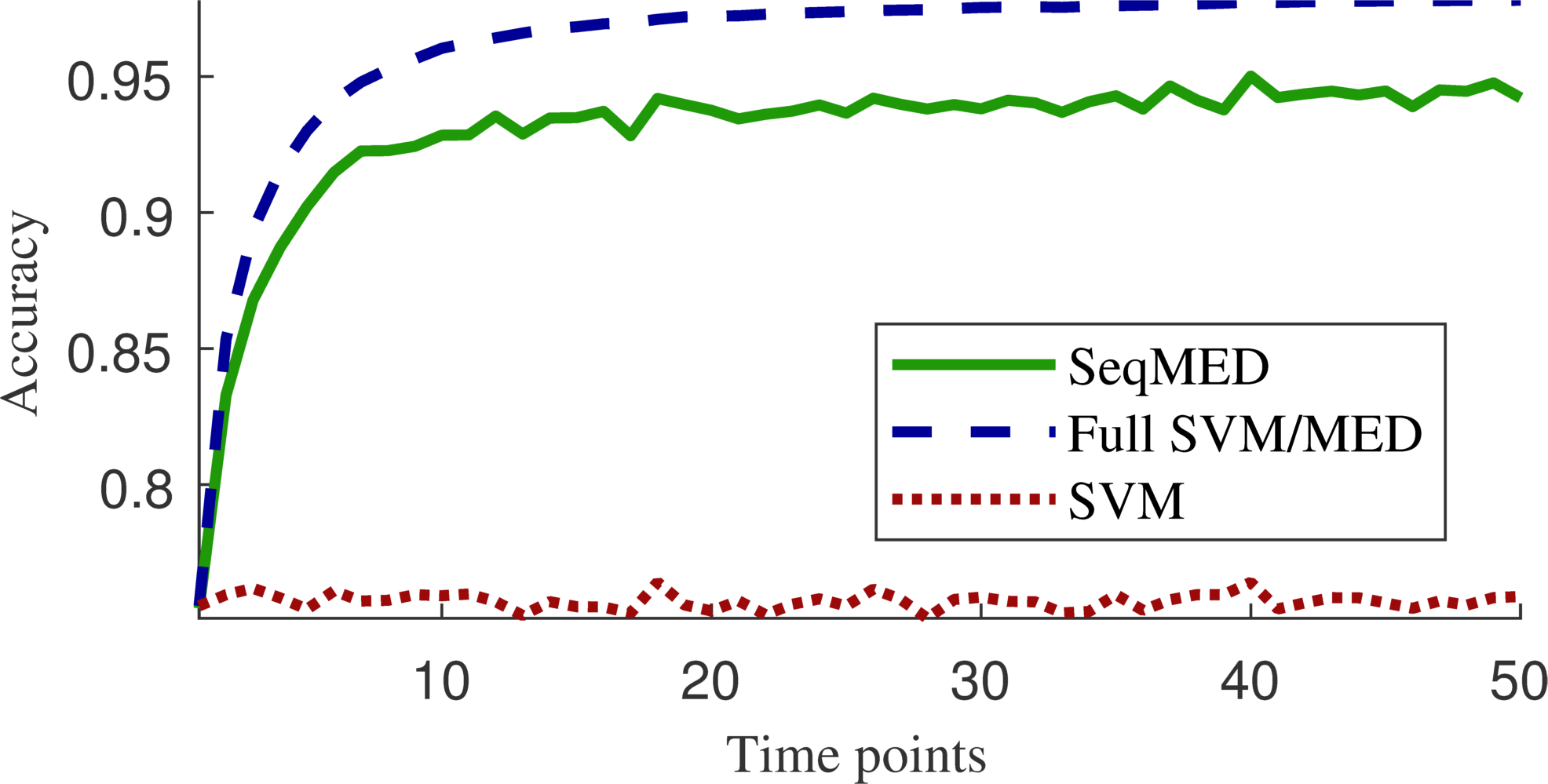}}
	\end{minipage}
		\vspace{-20pt}
	\caption{Accuracy of prediction for categorical fully labeled simulated data. The proposed sequential MED (SeqMED) classifier performs almost as well as the full batch implementation of the SVM/MED (Full SVM/MED). }
	\label{fig:super}
\end{figure}

In the second simulation, we generate data from the interior of a 3-dimensional sphere where one class is roughly at the center of the sphere and the other class is on the shell, but only 10\% of the samples are labeled. We use a rbf kernel with width 1 for the kernel function and a heat kernel with width 0.01 and a 20 nearest neighbors graph for the graph Laplacian. Figure \ref{fig:semisuper} shows improvement in performance of the sequential model similar to in Figure \ref{fig:super}. We use the approximate kernel function of Subsection \ref{approx_k_func} to perform each update, establishing that the approximation is adequate.

\begin{figure}[htb]
	\begin{minipage}[b]{1.0\linewidth}
		\centering
		\centerline{\includegraphics[width=8.5cm]{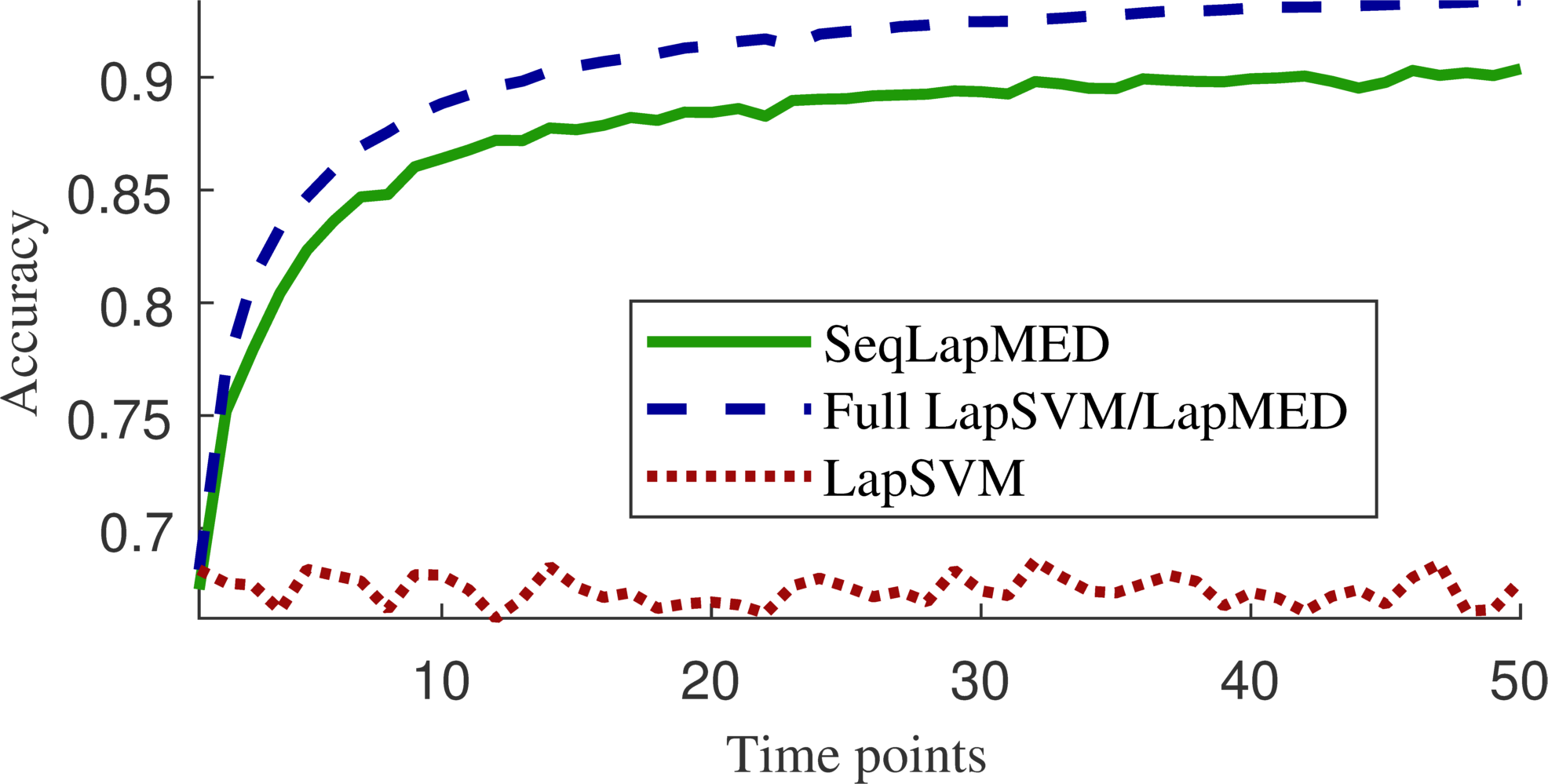}}
	\end{minipage}
	\vspace{-20pt}
	\caption{Accuracy of prediction for continuous simulated data with 10\% labeled.}
	\label{fig:semisuper}
\end{figure}

\subsection{Data}

We compare the proposed algorithms on the Isolet speech database from the UCI machine learning repository \cite{Lichman:2013} following the experimental framework used in \cite{Belkin:2006:MRG:1248547.1248632}. To train the models, we take the entire training set of 120 speakers (isolet1 - isolet4) and break them into 24 groups (time points) of 5 speakers where only the first speaker is labeled. At each time point, the models train on 260 samples ($t=21$ and $23$ only have 259) where 52 of the samples are labeled. The parameters are set in the same way as in \cite{Belkin:2006:MRG:1248547.1248632} and the test set is similarly composed of the 1,559 samples from isolet5. Figure \ref{fig:data} shows that, after two time points, the sequential model always performs better than the model that ignores previous data, and comes close to performing as well as the fully re-trained model as time progresses.

\begin{figure}[htb]
	\begin{minipage}[b]{1.0\linewidth}
		\centering
		\centerline{\includegraphics[width=8.5cm]{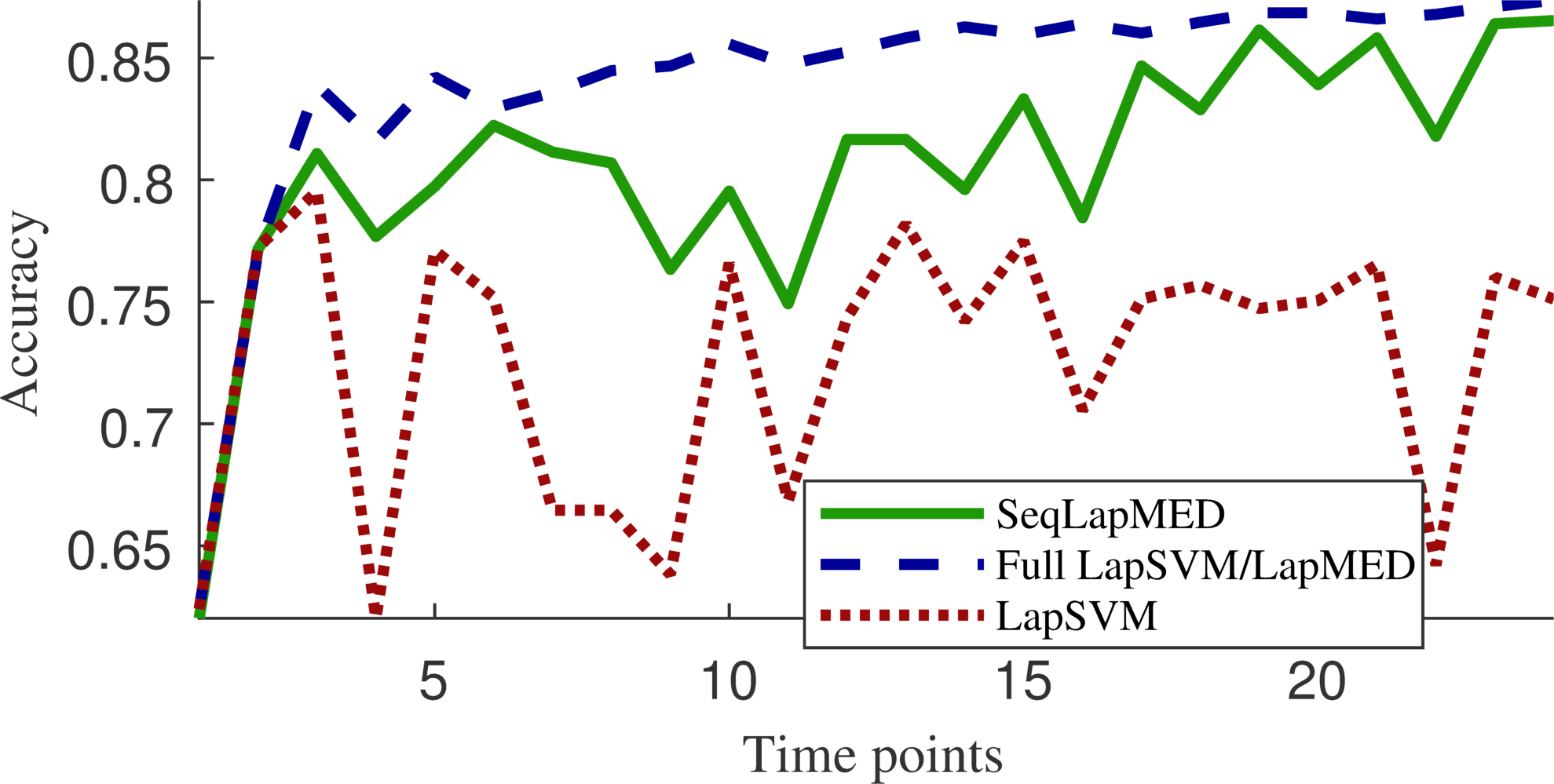}}
	\end{minipage}
	\vspace{-20pt}
	\caption{Accuracy of prediction on isolet5 for models trained on partially labeled speech isolets 1-4. The proposed semi-supervised sequential Laplacian MED classifier (SeqLapMED) comes close to the full Laplacian SVM \cite{Belkin:2006:MRG:1248547.1248632} as time progresses. }
	\label{fig:data}
\end{figure}

\section{Conclusions}

We have proposed recursive versions of supervised and semi-supervised maximum margin classifiers in the minimum entropy discrimination (MED) classification framework. The proposed sequential maximum margin classifiers perform nearly as well as a much more computationally expensive fully re-trained maximum margin classifiers and significantly better than a classifier that ignores previous data.

\appendix
\section{Appendix}

\begin{proof}[Proof of Theorem \ref{thm:SeqMED}]
Let $\bm{\mu}_{(\tau-1)} = \sum_{t=1}^{\tau-1} f(\bm{X}_{(t)})^T \bm{Y}_{(t)} \hat{\bm{\alpha}}_{(t)}$ where $\bm{\mu}_{(0)} = \bm{0}$. At time $\tau$, let the priors be $\bm{\theta} \sim N(\bm{\mu}_{(\tau-1)}, \bm{\I})$, $b \sim N(0, \sigma^2) $ where $ \sigma^2 \rightarrow \infty $, and $\gamma_i \sim C_{(\tau)} e^{-C_{(\tau)} (1-\gamma_i)} \mathcal{I}(\gamma_i \le 1) $. Then the posterior $\P(\bm{\theta}, b, \bm{\gamma}| \left\{ \mathcal{D} \right\}_{t=1}^{\tau}) $
\begin{flalign*}
& = \frac{\P_0(\bm{\theta}) \P_0(b) \P_0(\bm{\gamma} )}{Z(\hat{\bm{\alpha}}_{(\tau)})} \exp  \left\{ \sum_{i=1}^{n_{(\tau)}}  \hat{\alpha}_{(\tau)i} \left (y_{(\tau)i} (f(\bm{X}_{(\tau)}) \bm{\theta} + b) - \gamma_i \right) \right\} \\
& = \frac{\P_0(\bm{\theta})}{Z_\theta (\hat{\bm{\alpha}}_{(\tau)})} \exp \left\{ \sum_{i=1}^{n_{(\tau)}}  \hat{\alpha}_{(\tau)i} y_{(\tau)i} f(\bm{X}_{(\tau)}) \bm{\theta} \right\}  \\
& \hspace{12pt} \frac{\P_0(\bm{b})}{Z_b(\hat{\bm{\alpha}}_{(\tau)})} \exp \left\{ b \sum_{i=1}^{n_{(\tau)}}  y_{(\tau)i} \hat{\alpha}_{(\tau)i} \right\} \hspace{-2pt} \frac{\prod_{i=1}^{n_{(\tau)}}  \P_0(\gamma_i ) }{Z_{\gamma_i}(\hat{\bm{\alpha}}_{(\tau)})} e^{- \sum_{i=1}^{n_{(\tau)}}  \hat{\alpha}_{(\tau)i} \gamma_i} \\
& = \P(\bm{\theta}| \bm{X}_{(1)}, \bm{y}_{(1)}, \dots, \bm{X}_{(\tau)}, \bm{y}_{(\tau)}) \P(\bm{b}| \bm{y}_{(1)}, \dots, \bm{y}_{(\tau)}) \P(\bm{\gamma}).
\end{flalign*} 

\noindent So the posterior of the weights $ \P(\bm{\theta}| \bm{X}_{(1)}, \bm{y}_{(1)}, \dots, \bm{X}_{(\tau)}, \bm{y}_{(\tau)}) $
\begin{flalign*}
& = \frac{\exp \left\{ -0.5 (\bm{\theta}-\bm{\mu}_{(\tau-1)})^T (\bm{\theta}-\bm{\mu}_{(\tau-1)}) + \hat{\bm{\alpha}}_{(\tau)}^T \bm{Y}_{(\tau)} f(\bm{X}_{(\tau)}) \bm{\theta} \right\} }{ (2\pi )^{p/2} Z(\hat{\bm{\alpha}}_{(\tau)}) } \\
& = \frac{ \frac{\exp \left\{ -0.5 \left( \bm{\theta}^T \bm{\theta} -2\bm{\mu}_{(\tau-1)}^T \bm{\theta} - 2\hat{\bm{\alpha}}_{(\tau)}^T \bm{Y}_{(\tau)} f(\bm{X}_{(\tau)}) \bm{\theta} \right)  \right\}}{(2\pi )^{p/2}} }{ \int \frac{ \exp \left\{ -0.5 \left( \bm{\theta}^T \bm{\theta} -2\bm{\mu}_{(\tau-1)}^T \bm{\theta} - 2\hat{\bm{\alpha}}_{(\tau)}^T \bm{Y}_{(\tau)} f(\bm{X}_{(\tau)}) \bm{\theta} \right)  \right\}}{(2\pi )^{p/2}} d\bm{\theta}} \\
& = \exp \left\{ -0.5 \left( \bm{\theta}  -  (\bm{\mu}_{(\tau-1)} + f(\bm{X}_{(\tau)})^T \bm{Y}_{(\tau)} \hat{\bm{\alpha}}_{(\tau)} )\right) ^T \right. \\
& \left. \hspace{53pt} \left( \bm{\theta}  - (\bm{\mu}_{(\tau-1)} + f(\bm{X}_{(\tau)})^T \bm{Y}_{(\tau)} \hat{\bm{\alpha}}_{(\tau)} ) \right) \right\} \big/ (2\pi )^{p/2} \\
& \sim N( \bm{\mu}_{(\tau-1)} + f(\bm{X}_{(\tau)})^T \bm{Y}_{(\tau)} \hat{\bm{\alpha}}_{(\tau)} , \bm{\I}),
\end{flalign*}
the posterior of the bias term $\P(\bm{b}| \bm{y}_{(1)}, \dots, \bm{y}_{(\tau)})$
\begin{flalign*}
& = \frac{ (2\pi \sigma^2)^{-1/2}\exp \left\{ -0.5 ( b^2 - 2 \sigma^2 b \bm{y}_{(\tau)} ^T \hat{\bm{\alpha}}_{(\tau)} )/\sigma^2 \right\} }{\int  (2\pi \sigma^2)^{-1/2}\exp \left\{ -0.5 ( b^2 - 2 \sigma^2 b \bm{y}_{(\tau)} ^T \hat{\bm{\alpha}}_{(\tau)} )/\sigma^2 \right\} db}\\
& = \frac{e^{ -0.5 (b - \sigma^2 \bm{y}_{(\tau)} ^T \hat{\bm{\alpha}}_{(\tau)} )^2/\sigma^2}}{\sqrt{2\pi \sigma^2}} \sim N(\sigma^2 \bm{y}_{(\tau)} ^T \hat{\bm{\alpha}}_{(\tau)} , \sigma^2)  \\
& \Rightarrow \text{if } \sigma \rightarrow \infty, \text{then } N(\sigma^2 \bm{y}_{(\tau)} ^T \hat{\bm{\alpha}}_{(\tau)} , \sigma^2)  \rightarrow N(0, \infty) \\
& \text{ as long as the optimal Lagrange multipliers satisfy } \bm{y}_{(\tau)}^T \hat{\bm{\alpha}}_{(\tau)}  = 0,
\end{flalign*}
and the posterior of the margin parameters $ \P(\bm{\gamma}) $ do not depend on the data.

\end{proof}

\begin{proof}[Proof of Corollary \ref{coll:kern}]
At time $\tau$, let $\bm{\omega} = f(\bm{X}_{(\tau)}) \bm{\theta}$ have prior $ N(\bm{\mu}_{(\tau-1)}, \bm{K}_{(\tau)}) $ where $ \bm{\mu}_{(\tau-1)} = \sum_{t=1}^{\tau-1} k(\bm{X}_{(\tau)} , \bm{X}_{(t)}) \bm{Y}_{(t)} \hat{\bm{\alpha}}_{(t)} $. Then the posterior $ \P(\bm{\omega}| \bm{X}_{(1)}, \bm{y}_{(1)}, \dots, \bm{X}_{(\tau)}, \bm{y}_{(\tau)}) $
\begin{flalign*}
& = \frac{\P_0(\bm{\omega})}{Z_\omega (\hat{\bm{\alpha}}_{(\tau)})} \exp \left\{ \sum_{i=1}^{n_{(\tau)}}  \hat{\alpha}_{(\tau)i} y_{(\tau)i} \bm{\omega} \right\} \ \\
& = \frac{\exp \left\{ -0.5 ( \bm{\omega}  -  \bm{\mu}_{(\tau-1)})^T \bm{K}_{(\tau)}^{-1} (\bm{\omega}  -  \bm{\mu}_{(\tau-1)})  +\hat{\bm{\alpha}}_{(\tau)}^{T} \bm{Y}_{(\tau)} \bm{\omega} \right\} }{ |2\pi \bm{K}_{(\tau)}|^{1/2} Z_\omega (\hat{\bm{\alpha}}_{(\tau)}) } \\
& = \frac{e^{-0.5 \left( \bm{\omega}  -  (\bm{\mu}_{(\tau-1)} + \bm{K}_{(\tau)} \bm{Y}_{(\tau)} \hat{\bm{\alpha}}_{(\tau)} ) \right)^T \bm{K}_{(\tau)}^{-1} \left(\bm{\omega}  -  (\bm{\mu}_{(\tau-1)} + \bm{K}_{(\tau)} \bm{Y}_{(\tau)} \hat{\bm{\alpha}}_{(\tau)} ) \right) } }{  |2\pi \bm{K}_{(\tau)}|^{1/2} } \\
& \sim N(\bm{\mu}_{(\tau-1)} + \bm{K}_{(\tau)} \bm{Y}_{(\tau)} \hat{\bm{\alpha}}_{(\tau)}  , \bm{K}_{(\tau)}).
\end{flalign*}
\end{proof}

\begin{proof}[Proof of Corollary \ref{coll:Opt_Lagrange}]
The optimal Lagrange multipliers at $t = \tau$ are the solution to $\underset{\bm{\alpha}_{(\tau)}}{\arg\max} -\log \left( Z(\bm{\alpha}_{(\tau)}) \right)$
$$ = \underset{\bm{\alpha}_{(\tau)}}{\arg\max} -\log \left( Z_\theta(\bm{\alpha}_{(\tau)}) \right)  -\log \left( Z_b(\bm{\alpha}_{(\tau)}) \right) -\log \left( Z_\gamma (\bm{\alpha}_{(\tau)}) \right) $$
or
$$ = \underset{\bm{\alpha}_{(\tau)}}{\arg\max} -\log \left( Z_\omega(\bm{\alpha}_{(\tau)}) \right)  -\log \left( Z_b(\bm{\alpha}_{(\tau)}) \right) -\log \left( Z_\gamma (\bm{\alpha}_{(\tau)}) \right) $$
where $ -\log \left( Z_\theta (\bm{\alpha}_{(\tau)}) \right) $
\begin{flalign*}
& = -\log \left( \int \frac{ e^{\bm{\alpha}_{(\tau)}^T \bm{Y}_{(\tau)} f(\bm{X}_{(\tau)}) \bm{\theta} - 0.5 (\bm{\theta}-\bm{\mu}_{(\tau-1)})^T (\bm{\theta}-\bm{\mu}_{(\tau-1)}) }}{(2\pi )^{p/2}} d\bm{\theta} \right) \\
& =  -\bm{\alpha}_{(\tau)}^{T} \bm{Y}_{(\tau)} f(\bm{X}_{(\tau)}) \bm{\mu}_{(\tau-1)} \\
& \hspace{10pt} - 0.5 \bm{\alpha}_{(\tau)}^{T} \bm{Y}_{(\tau)} f(\bm{X}_{(\tau)}) f(\bm{X}_{(\tau)})^T \bm{Y}_{(\tau)} \bm{\alpha}_{(\tau)},
\end{flalign*}
$-\log \left( Z_\omega (\bm{\alpha}_{(\tau)})  \right) $
\begin{flalign*}
& = -\log \left( \int \frac{ e^{\bm{\alpha}_{(\tau)}^{T} \bm{Y}_{(\tau)} \bm{\omega}  -0.5 ( \bm{\omega}  -  \bm{\mu}_{(\tau-1)})^T \bm{K}_{(\tau)}^{-1} (\bm{\omega}  -  \bm{\mu}_{(\tau-1)}) } }{ |2\pi \bm{K}_{(\tau)}|^{1/2} }  d\bm{\omega} \right)  \\
& =  -\bm{\alpha}_{(\tau)}^{T} \bm{Y}_{(\tau)} \bm{\mu}_{(\tau-1)} - 0.5 \bm{\alpha}_{(\tau)}^{T} \bm{Y}_{(\tau)} \bm{K}_{(\tau)} \bm{Y}_{(\tau)} \bm{\alpha}_{(\tau)}, 
\end{flalign*}
$ -\log \left( Z_b (\bm{\alpha}_{(\tau)})  \right) $
\begin{flalign*}
& = -\log \left(\int \frac{ e^{ -0.5 (b - \sigma^2 \bm{y}_{(\tau)}^{T} \bm{\alpha}_{(\tau)} )^2/\sigma^2} }{ \sqrt{2\pi \sigma^2}}  db \right) -\log \left( e^{0.5 \sigma^2 (\bm{y}_{(\tau)}^{T} \bm{\alpha}_{(\tau)} )^2} \right) \\
& = -0.5 \sigma^2 (\bm{y}_{(\tau)}^{T} \bm{\alpha}_{(\tau)} )^2  \Rightarrow \text{if } \sigma \rightarrow \infty, \text{ then }\bm{y}_{(\tau)}^{T} \bm{\alpha}_{(\tau)} = 0
\end{flalign*}
and $  -\log \left( Z_\gamma (\bm{\alpha}_{(\tau)})  \right) = - \sum_{i=1}^{n_{(\tau)}} \log \left( Z_{\gamma_i} (\bm{\alpha}_{(\tau)})  \right) $
\begin{flalign*} 
& = - \sum_{i=1}^{n_{(\tau)}} \log \left( \int_{-\infty}^{1} C_{(\tau)} e^{-C_{(\tau)}(1-\gamma_i)} e^{- \alpha_{(\tau)i} \gamma_i} \, d\gamma_i \right) \\
& = - \sum_{i=1}^{n_{(\tau)}}\log \left( \left. \frac{C_{(\tau)}}{C_{(\tau)} - \alpha_{(\tau)i}} e^{-C_{(\tau)} + \gamma_i (C_{(\tau)} - \alpha_{(\tau)i})} \right|^1_{-\infty} \right) \\
& = - \sum_{i=1}^{n_{(\tau)}} \log \left(\frac{C_{(\tau)} e^{-\alpha_{(\tau)i}}  }{C_{(\tau)} - \alpha_{(\tau)i}} \right) = \sum_{i=1}^{n_{(\tau)}} \alpha_{(\tau)i} + \log \left(1 - \frac{\alpha_{(\tau)i} }{ C_{(\tau)} } \right) .
\end{flalign*} 
\end{proof}

\begin{proof}[Proof of Theorem \ref{thm:SeqLapMED}]
At time $\tau$, let the priors for $b$ and $\gamma_i$ be the same as in Theorem \ref{thm:SeqMED}, $ \lambda \sim Exp.(\nu) $ where $\nu \rightarrow \infty$, and $\bm{\theta} \, (\text{or }\omega) \sim N \left(\bm{\mu}_{(\tau-1)}, \bm{\Sigma}_{(\tau-1)} \right) $. Then the posterior $ \P(\bm{\theta}, b, \bm{\gamma}, \lambda| \left\{ \mathcal{D} \right\}_{t=1}^{\tau})  $ and partition function $ Z_\theta(\bm{\alpha}_{(\tau)}, \beta_{(\tau)})$ factorize similarly as
$$ \P(\bm{\theta}| \bm{X}_{(1)}, \bm{y}_{(1)}, \dots, \bm{X}_{(\tau)}, \bm{y}_{(\tau)}) \P(\bm{b}| \bm{y}_{(1)}, \dots, \bm{y}_{(\tau)}) \P(\bm{\gamma}) \P(\lambda) $$
and 
$$ Z_\theta(\bm{\alpha}_{(\tau)}, \beta_{(\tau)}) Z_{b}(\bm{\alpha}_{(\tau)}) Z_\lambda(\beta_{(\tau)}) \prod_{i=1}^{l_{(\tau)}} Z_{\gamma_i}(\bm{\alpha}_{(\tau)}) .$$
The bias and margin terms are independent of $\beta_{(\tau)} $, so their posterior and partition functions are the same as in Theorem \ref{thm:SeqMED}. The posterior of the smoothness parameter $\lambda$ does not depend on the data and $  -\log \left( Z_\lambda(\beta_{(\tau)}) \right) $
\begin{flalign*} 
& = -\log \left( \int_{0}^{\infty} \nu e^{-\nu \lambda} e^{\beta_{(\tau)} \lambda} \, d\lambda \right) = -\log \left( \frac{\nu}{\nu - \beta_{(\tau)}} \right) \\
&  \Rightarrow \text{ if } \nu \rightarrow \infty, \text{ then } \log(1 - \beta_{(\tau)} / \nu) = 0 .
\end{flalign*}
Let the parameters for the prior distribution of $\bm{\theta}$ be
$$ 
\bm{\mu}_{(\tau-1)} = \bm{G}_{(\tau-1)}^{-1} \sum_{t=1}^{\tau-1} f(\bm{X}_{(t)})^T \bm{J}^T  \bm{Y}_{(t)}  \hat{\bm{\alpha}}_{(t)}, \,\, \bm{\Sigma}_{(\tau-1)} = \bm{G}_{(\tau-1)}^{-1}, 
$$
where $ \bm{G}_{(\tau-1)} = \bm{G}_{(\tau-2)} + 2\beta_{(\tau-1)} f(\bm{X}_{(\tau-1)})^T \bm{L}_{(\tau-1)} f(\bm{X}_{(\tau-1)}) $, and $\bm{G}_{(0)} = \bm{\I}$, then
the posterior $ \P(\bm{\theta}| \bm{X}_{(1)}, \bm{y}_{(1)}, \dots, \bm{X}_{(\tau)}, \bm{y}_{(\tau)}) $
\begin{flalign*}
& =  \frac{\exp \left\{ -0.5 (\bm{\theta} - \bm{\mu}_{(\tau-1)})^T \bm{\Sigma}_{(\tau-1)}^{-1} (\bm{\theta} - \bm{\mu}_{(\tau-1)} )  \right\}}{\det(2\pi \bm{\Sigma}_{(\tau-1)})^{1/2}} \\
& \hspace{12pt} \frac{\exp \left\{ \hat{\bm{\alpha}}_{(\tau)}^T\bm{Y}_{(\tau)} \bm{J} f(\bm{X}_{(\tau)}) \bm{\theta} - \beta_{(\tau)} \bm{\theta}^T f(\bm{X}_{(\tau)})^T \bm{L}_{(\tau)} f(\bm{X}_{(\tau)}) \bm{\theta} \right\} }{ Z_\theta(\hat{\bm{\alpha}}_{(\tau)}, \beta_{(\tau)} ) } \\
& = \exp \left\{ -0.5 \left( \bm{\theta}^T  \bm{G}_{(\tau-1)} \bm{\theta} - 2 \bm{\theta}^T \sum_{t=1}^{\tau-1} f(\bm{X}_{(t)})^T \bm{J}^T  \bm{Y}_{(t)}  \hat{\bm{\alpha}}_{(t)} \right. \right. \\
& + \hspace{-2pt} \left(\sum_{t=1}^{\tau-1} f(\bm{X}_{(t)})^T \hspace{-2pt} \bm{J}^T  \bm{Y}_{(t)}  \hat{\bm{\alpha}}_{(t)} \hspace{-3pt} \right)^T \hspace{-8pt} \bm{G}_{(\tau-1)}^{-1} \hspace{-2pt} \left(\sum_{t=1}^{\tau-1} f(\bm{X}_{(t)})^T \hspace{-2pt} \bm{J}^T  \bm{Y}_{(t)}  \hat{\bm{\alpha}}_{(t)} \hspace{-3pt} \right) \\
& - 2 \hat{\bm{\alpha}}_{(\tau)}^T\bm{Y}_{(\tau)} \bm{J} f(\bm{X}_{(\tau)}) \bm{\theta} +2 \beta_{(\tau)} \bm{\theta}^T f(\bm{X}_{(\tau)})^T \bm{L}_{(\tau)} f(\bm{X}_{(\tau)}) \bm{\theta} \Bigg) \Bigg\} \\
& \bigg/ \left( \det(2\pi \bm{G}_{(\tau-1)}^{-1})^{1/2}   Z_\theta(\hat{\bm{\alpha}}_{(\tau)}, \beta_{(\tau)}) \right) \\
& = \exp \left\{ -0.5 \left( \bm{\theta}^T \bm{G}_{(\tau)} \bm{\theta} - 2 \bm{\theta}^T \sum_{t=1}^{\tau} f(\bm{X}_{(t)})^T \bm{J}^T  \bm{Y}_{(t)}  \hat{\bm{\alpha}}_{(t)} \right.\right. \\
& + \hspace{-2pt} \left(\sum_{t=1}^{\tau} f(\bm{X}_{(t)})^T \hspace{-2pt} \bm{J}^T  \bm{Y}_{(t)}  \hat{\bm{\alpha}}_{(t)} \hspace{-3pt} \right)^T \hspace{-8pt} \bm{G}_{(\tau)}^{-1} \hspace{-2pt} \left(\sum_{t=1}^{\tau} f(\bm{X}_{(t)})^T \hspace{-2pt} \bm{J}^T  \bm{Y}_{(t)}  \hat{\bm{\alpha}}_{(t)} \hspace{-3pt} \right) \hspace{-3pt} \Bigg) \\
& - \hspace{-2pt} \left(\sum_{t=1}^{\tau-1} f(\bm{X}_{(t)})^T \hspace{-2pt} \bm{J}^T  \bm{Y}_{(t)}  \hat{\bm{\alpha}}_{(t)} \hspace{-3pt} \right)^T \hspace{-4pt} \frac{ \bm{G}_{(\tau-1)}^{-1} }{2} \hspace{-2pt} \left(\sum_{t=1}^{\tau-1} f(\bm{X}_{(t)})^T \hspace{-2pt} \bm{J}^T  \bm{Y}_{(t)}  \hat{\bm{\alpha}}_{(t)} \hspace{-3pt} \right) \\
& + \hspace{-2pt} \left(\sum_{t=1}^{\tau} f(\bm{X}_{(t)})^T \hspace{-2pt} \bm{J}^T  \bm{Y}_{(t)}  \hat{\bm{\alpha}}_{(t)} \hspace{-3pt} \right)^T \hspace{-4pt} \frac{ \bm{G}_{(\tau)}^{-1} }{2} \hspace{-2pt} \left(\sum_{t=1}^{\tau} f(\bm{X}_{(t)})^T \hspace{-2pt} \bm{J}^T  \bm{Y}_{(t)}  \hat{\bm{\alpha}}_{(t)} \hspace{-3pt} \right)
 \hspace{-3pt} \Bigg) \hspace{-2pt} \Bigg\} \\
& \hspace{6pt} \bigg/ \left( \det(2\pi \bm{G}_{(\tau-1)}^{-1})^{1/2}   Z_\theta(\hat{\bm{\alpha}}_{(\tau)}, \beta_{(\tau)}) \right) \\
& = \exp \left\{ -0.5 \left( \bm{\theta} - \bm{G}_{(\tau)}^{-1} \sum_{t=1}^{\tau} f(\bm{X}_{(t)})^T \bm{J}^T  \bm{Y}_{(t)}  \hat{\bm{\alpha}}_{(t)} \right)^T \bm{G}_{(\tau)} \right. \\
& \left. \left( \bm{\theta} - \bm{G}_{(\tau)}^{-1} \sum_{t=1}^{\tau} f(\bm{X}_{(t)})^T \bm{J}^T  \bm{Y}_{(t)}  \hat{\bm{\alpha}}_{(t)} \right) \right\} \bigg/ \det(2\pi \bm{G}_{(\tau)}^{-1})^{1/2} \\
& \sim N\left( \bm{G}_{(\tau)}^{-1} \sum_{t=1}^{\tau} f(\bm{X}_{(t)})^T \bm{J}^T  \bm{Y}_{(t)}  \hat{\bm{\alpha}}_{(t)}, (\bm{G}_{(\tau)})^{-1}  \right)
\end{flalign*}
and $ -\log \left( Z_\theta (\bm{\alpha}_{(\tau)}), \beta_{(\tau)}  \right) $
\begin{flalign*}
& = -\log \left( \det(2\pi \bm{G}_{(\tau)}^{-1})^{1/2} / \det(2\pi \bm{G}_{(\tau-1)}^{-1})^{1/2} \right) \\
& - 0.5 \Bigg( \hspace{-4pt} \left(\sum_{t=1}^{\tau} f(\bm{X}_{(t)})^T \bm{J}^T  \bm{Y}_{(t)}  \bm{\alpha}_{(t)} \hspace{-2pt} \right)^T \hspace{-8pt} \bm{G}_{(\tau)}^{-1} \hspace{-2pt} \left(\sum_{t=1}^{\tau} f(\bm{X}_{(t)})^T \bm{J}^T  \bm{Y}_{(t)}  \bm{\alpha}_{(t)} \hspace{-2pt} \right) \\
& + \left(\sum_{t=1}^{\tau-1} f(\bm{X}_{(t)})^T \bm{J}^T  \bm{Y}_{(t)}  \bm{\alpha}_{(t)} \hspace{-2pt} \right)^T \hspace{-8pt} \bm{G}_{(\tau-1)}^{-1} \hspace{-2pt} \left(\sum_{t=1}^{\tau-1} f(\bm{X}_{(t)})^T \bm{J}^T  \bm{Y}_{(t)}  \bm{\alpha}_{(t)} \hspace{-2pt} \right) \hspace{-4pt} \Bigg)
\end{flalign*}
\begin{flalign*}
& = 0.5 \Bigg( \hspace{-3pt} \log \left(\det(2\pi \bm{G}_{(\tau-1)}^{-1}) \right) + \left(\sum_{t=1}^{\tau-1} f(\bm{X}_{(t)})^T \bm{J}^T  \bm{Y}_{(t)} \bm{\alpha}_{(t)} \right)^T  \\
& \hspace{26pt} \left( \hspace{-3pt} \bm{G}_{(\tau-1)} \left( 2\beta_{(\tau)} f(\bm{X}_{(\tau)})^T \hspace{-2pt} \bm{L}_{(\tau)} f(\bm{X}_{(\tau)}) \right)^{-1} \hspace{-10pt}\bm{G}_{(\tau-1)} \hspace{-2pt} + \hspace{-1pt}\bm{G}_{(\tau-1)} \hspace{-1pt} \right)^{-1} \\
& \hspace{26pt} \left(\sum_{t=1}^{\tau-1} f(\bm{X}_{(t)})^T \bm{J}^T  \bm{Y}_{(t)} \bm{\alpha}_{(t)} \right) -\log \left( \det(2\pi \bm{G}_{(\tau)}^{-1}) \right) \hspace{-3pt} \Bigg) \\
& - 0.5 \, \bm{\alpha}_{(\tau)}^T\bm{Y}_{(\tau)} \bm{J} f(\bm{X}_{(\tau)}) \bm{G}_{(\tau)}^{-1} f(\bm{X}_{(\tau)})^T \bm{J}^T  \bm{Y}_{(\tau)} \bm{\alpha}_{(\tau)} \\
& - \bm{\alpha}_{(\tau)}^T\bm{Y}_{(\tau)} \bm{J} f(\bm{X}_{(\tau)}) \bm{G}_{(\tau)}^{-1}  \sum_{t=1}^{\tau-1} f(\bm{X}_{(t)})^T \bm{J}^T  \bm{Y}_{(t)}  \bm{\alpha}_{(t)} \\
& = 0.5 \left( Const._{\beta_{(\tau)}} \hspace{-4pt} - \bm{\alpha}_{(\tau)}^T\bm{Y}_{(\tau)} \bm{J} f(\bm{X}_{(\tau)}) \bm{G}_{(\tau)}^{-1} f(\bm{X}_{(\tau)})^T \bm{J}^T  \bm{Y}_{(\tau)} \bm{\alpha}_{(\tau)} \right) \\
& - \bm{\alpha}_{(\tau)}^T\bm{Y}_{(\tau)} \bm{J} f(\bm{X}_{(\tau)}) \bm{G}_{(\tau)}^{-1}  \sum_{t=1}^{\tau-1} f(\bm{X}_{(t)})^T \bm{J}^T  \bm{Y}_{(t)}  \bm{\alpha}_{(t)} 
\end{flalign*}
where $Const._{\beta_{(\tau)}} $ can be dropped from the objective when $\beta_{(t)}$ are fixed parameters. \\
Or let the parameters for the prior distribution of $\bm{\omega} = f(\bm{X}_{(\tau)}) \bm{\theta}$ be
\begin{flalign*}
& \bm{\mu}_{(\tau-1)} = \sum_{t=1}^{\tau-1} k_{(\tau-1)} ( \bm{X}_{(\tau)}, \bm{X}_{(t)} ) \bm{J}^T \bm{Y}_{(t)}  \hat{\bm{\alpha}}_{(t)} \\
& \bm{\Sigma}_{(\tau-1)} = k_{(\tau-1)} ( \bm{X}_{(\tau)}, \bm{X}_{(\tau)} ) \text{ where } k_{(0)}( \bm{x}, \bm{x}') = \langle f(\bm{x}), f(\bm{x}') \rangle \\
& k_{(\tau-1)}( \bm{x}, \bm{x}') = k_{(\tau-2)}(\bm{x}, \bm{x}') - k_{(\tau-2)}(\bm{x}, \bm{X}_{(\tau-1)}) \\
& \hspace{-4pt} \left( \hspace{-2pt} \left(2 \beta_{(\tau-1)}\bm{L}_{(\tau-1)}\right)^{-1} + k_{(\tau-2)} \left(\bm{X}_{(\tau-1)}, \bm{X}_{(\tau-1)} \right) \hspace{-2pt} \right)^{-1} \hspace{-10pt} k_{(\tau-2)}( \bm{X}_{(\tau-1)}, \bm{x}') .
\end{flalign*}
The posterior $ \P(\bm{\omega}| \bm{X}_{(1)}, \bm{y}_{(1)}, \dots, \bm{X}_{(\tau)}, \bm{y}_{(\tau)}) $
\begin{flalign*}
& =  \frac{e^{-0.5 (\bm{\omega} - \bm{\mu}_{(\tau-1)})^T \bm{\Sigma}_{(\tau-1)}^{-1} (\bm{\omega} - \bm{\mu}_{(\tau-1)} ) }}{\det(2\pi \bm{\Sigma}_{(\tau-1)})^{1/2}} \frac{e^{\hat{\bm{\alpha}}_{(\tau)}^T\bm{Y}_{(\tau)} \bm{J} \bm{\omega} - \beta_{(\tau)} \bm{\omega}^T \bm{L}_{(\tau)} \bm{\omega} } }{ Z_\omega(\hat{\bm{\alpha}}_{(\tau)}, \beta_{(\tau)} ) } \\
& = \exp \bigg\{ \hspace{-2pt} -0.5 \bigg( \bm{\omega}^T  \left( k_{(\tau-1)} ( \bm{X}_{(\tau)}, \bm{X}_{(\tau)} )^{-1} + 2 \beta_{(\tau)}  \bm{L}_{(\tau)} \right) \bm{\omega} \\
& - \left. \left. \hspace{-2pt} 2 \bm{\omega}^T k_{(\tau-1)} ( \bm{X}_{(\tau)}, \bm{X}_{(\tau)} )^{-1} \Bigg( \sum_{t=1}^{\tau-1} k_{(\tau-1)} (\bm{X}_{(\tau)}, \bm{X}_{(t)}) \bm{J}^T  \bm{Y}_{(t)}  \hat{\bm{\alpha}}_{(t)} \right. \right. \\
& \hspace{120pt} + k_{(\tau-1)} ( \bm{X}_{(\tau)}, \bm{X}_{(\tau)} ) \bm{J}^T \bm{Y}_{(\tau)} \hat{\bm{\alpha}}_{(\tau)} \hspace{-2pt} \Bigg) \\
& + \hspace{-2pt} \left( \sum_{t=1}^{\tau-1} k_{(\tau-1)} ( \bm{X}_{(\tau)}, \bm{X}_{(t)} )  \bm{J}^T  \bm{Y}_{(t)}  \hat{\bm{\alpha}}_{(t)} \hspace{-2pt} \right)^T k_{(\tau-1)} ( \bm{X}_{(\tau)}, \bm{X}_{(\tau)} )^{-1}  \\
& \hspace{8pt} \left( \sum_{t=1}^{\tau-1} k_{(\tau-1)} ( \bm{X}_{(\tau)}, \bm{X}_{(t)} )  \bm{J}^T  \bm{Y}_{(t)}  \hat{\bm{\alpha}}_{(t)} \right) \bigg) \bigg\} \\
& \hspace{6pt} \bigg/ \left( \det \left(2\pi k_{(\tau-1)} ( \bm{X}_{(\tau)}, \bm{X}_{(\tau)} ) \right)^{1/2}   Z_\omega (\hat{\bm{\alpha}}_{(\tau)}, \beta_{(\tau)}) \right) \\
& = \exp \Bigg\{ \hspace{-4pt} -0.5 \left( \bm{\omega} - \sum_{t=1}^{\tau} k_{(\tau)} \left( \bm{X}_{(\tau)}, \bm{X}_{(t)} \right) \bm{J}^T  \bm{Y}_{(t)}  \hat{\bm{\alpha}}_{(t)} \right)^T  \\
& \hspace{12pt} k_{(\tau)} ( \bm{X}_{(\tau)}, \bm{X}_{(\tau)} )^{-1} \hspace{-2pt} \left( \bm{\omega} - \sum_{t=1}^{\tau} k_{(\tau)} \left( \bm{X}_{(\tau)}, \bm{X}_{(t)} \right) \bm{J}^T  \bm{Y}_{(t)}  \hat{\bm{\alpha}}_{(t)} \right) \hspace{-3pt} \Bigg\} \\
& \hspace{6pt} \bigg/ \det \left(2\pi  k_{(\tau)} ( \bm{X}_{(\tau)}, \bm{X}_{(\tau)} ) \right)^{1/2} \\
& \sim N\left( \sum_{t=1}^{\tau} k_{(\tau)} ( \bm{X}_{(\tau)}, \bm{X}_{(t)} ) \bm{J}^T  \bm{Y}_{(t)}  \hat{\bm{\alpha}}_{(t)}, k_{(\tau)} ( \bm{X}_{(\tau)}, \bm{X}_{(\tau)} ) \right)
\end{flalign*}
and  $ -\log \left( Z_\omega (\bm{\alpha}_{(\tau)}), \beta_{(\tau)}  \right) $
\begin{flalign*}
& = - 0.5 \Bigg( \log \left( \det \left(k_{(\tau)} ( \bm{X}_{(\tau)}, \bm{X}_{(\tau)} ) \, k_{(\tau-1)} ( \bm{X}_{(\tau)}, \bm{X}_{(\tau)} )^{-1} \right) \right) \\
& - \left(\sum_{t=1}^{\tau} k_{(\tau)} ( \bm{X}_{(\tau)}, \bm{X}_{(t)} ) \bm{J}^T  \bm{Y}_{(t)} \bm{\alpha}_{(t)} \right)^T k_{(\tau)} ( \bm{X}_{(\tau)}, \bm{X}_{(\tau)} )^{-1} \\
& \hspace{10pt} \left(\sum_{t=1}^{\tau} k_{(\tau)} ( \bm{X}_{(\tau)}, \bm{X}_{(t)} ) \bm{J}^T  \bm{Y}_{(t)} \bm{\alpha}_{(t)} \right) \\
& + \left(\sum_{t=1}^{\tau-1} k_{(\tau-1)} ( \bm{X}_{(\tau)}, \bm{X}_{(t)} ) \bm{J}^T  \bm{Y}_{(t)} \bm{\alpha}_{(t)} \right)^T \hspace{-6pt} k_{(\tau-1)} ( \bm{X}_{(\tau)}, \bm{X}_{(\tau)} )^{-1} \\
& \hspace{10pt}  \left(\sum_{t=1}^{\tau-1} k_{(\tau-1)} ( \bm{X}_{(\tau)}, \bm{X}_{(t)} ) \bm{J}^T  \bm{Y}_{(t)} \bm{\alpha}_{(t)} \right)  \Bigg) \\
& = 0.5 \left( Const._{\beta_{(\tau)}} \hspace{-4pt} - \bm{\alpha}_{(\tau)}^{T} \bm{Y}_{(\tau)} \bm{J} k_{(\tau)} ( \bm{X}_{(\tau)}, \bm{X}_{(\tau)} ) \bm{J}^T \bm{Y}_{(\tau)} \bm{\alpha}_{(\tau)} \right) \\
& - \bm{\alpha}_{(\tau)}^T\bm{Y}_{(\tau)} \bm{J} \sum_{t=1}^{\tau-1} k_{(\tau)} ( \bm{X}_{(\tau)}, \bm{X}_{(t)} ) \bm{J}^T  \bm{Y}_{(t)} \bm{\alpha}_{(t)} .
\end{flalign*}
\end{proof}

\bibliographystyle{IEEEbib}
\bibliography{refs} 

\end{document}